\documentclass[10pt,twocolumn,letterpaper]{article}

\usepackage[pagenumbers]{cvpr} 

\usepackage{graphicx}
\usepackage{amsmath}
\usepackage{amsthm}
\usepackage{amssymb}
\usepackage{booktabs}
\usepackage{algorithm}
\usepackage{algpseudocode}
\newcommand\scalemath[2]{\scalebox{#1}{\mbox{\ensuremath{\displaystyle #2}}}}
\newtheorem{theorem}{Theorem}

\usepackage[pagebackref,breaklinks,colorlinks]{hyperref}

\usepackage[capitalize]{cleveref}
\crefname{section}{Sec.}{Secs.}
\Crefname{section}{Section}{Sections}
\Crefname{table}{Table}{Tables}
\crefname{table}{Tab.}{Tabs.}

\def\cvprPaperID{*****} 
\def\confName{CVPR}
\def\confYear{2023}

\begin{document}

\title{Shadows Aren’t So Dangerous After All: A Fast and Robust Defense Against Shadow-Based Adversarial Attacks}

\author{Andrew Wang\\
Cornell University\\
{\tt\small aw632@cornell.edu}
\and
Wyatt Mayor\\
Monmouth College\\
{\tt\small wmayor@monmouthcollege.edu}
\and
Ryan Smith \\
University of Arizona \\
{\tt\small ryansmith1@arizona.edu} 
\and 
Gopal Nookula \\
U.C. Riverside \\
{\tt\small gnook001@ucr.edu}
\and 
Gregory Ditzler \\
Rowan University \\
{\tt\small ditzler@rowan.edu}}
\maketitle

\begin{abstract}
   Robust classification is essential in tasks like autonomous vehicle sign recognition, where the downsides of misclassification can be grave. Adversarial attacks threaten the robustness of neural network classifiers, causing them to consistently and confidently misidentify road signs. One such class of attack, shadow-based attacks, causes misidentifications by applying a natural-looking shadow to input images, resulting in road signs that appear natural to a human observer but confusing for these classifiers. Current defenses against such attacks use a simple adversarial training procedure to achieve a rather low 25\% and 40\% robustness on the GTSRB and LISA test sets, respectively. In this paper, we propose a robust, fast, and generalizable method, designed to defend against shadow attacks in the context of road sign recognition, that augments source images with binary adaptive threshold and edge maps. We empirically show its robustness against shadow attacks, and reformulate the problem to show its similarity $\varepsilon$ perturbation-based attacks. Experimental results show that our edge defense results in 78\% robustness while maintaining 98\% benign test accuracy on the GTSRB test set, with similar results from our threshold defense. 
   \footnote{Our code is available at \url{https://github.com/aw632/ShadowDefense}.}
\end{abstract}

\section{Introduction}
\label{sec:intro}

With the great success of neural networks in image classification has come the great vulnerability of adversarial examples---a class of examples designed to exploit the brittle nature of deep neural networks and fool models into making incorrect classifications by making small, human-imperceptible changes to the image. When these adversarial examples appear in mission-critical settings such as autonomous driving\cite{mahima2021adversarial}, medical imaging\cite{hirano2021universal}, and financial management\cite{Goldblum_2021}, the effects can be disastrous. The importance of defending against such examples and adversarial attacks has therefore spawned countermeasures, which have in turn spawned more advanced attacks, leading to a sort of adversarial arms race \cite{inevitable_2018}.
One recent adversarial attack, proposed by Zhong \etal \cite{zhong2022shadows}, involves darkening a section of the input image (``shadowing"), thereby causing misclassifications. Not only is this attack extremely effective against SOTA sign recognition models, achieving 90\% and 98\% attack success rates on the GTSRB and LISA benchmark datasets respectively, it is also realistic, requiring little to no specialized equipment and is easily unnoticeable by human drivers. 

We propose a new defense (see \cref{fig:overveiw}) against this attack based on adaptive threshold and edge maps that---even with no hyperparameter optimization---achieves 78\% robustness by trading off only roughly 1\% benign test accuracy. In specific, we:
\begin{itemize}
    \item motivate the use of adaptive threshold and edge maps with saliency maps, 
    \item demonstrate robustness, and effectiveness of our defense against shadow-based adversarial attacks,
    \item show its generalizability against classic gradient-based and non-gradient-based adversarial attacks,
    \item reformulate the shadows attack as an instance of $\varepsilon$ perturbations within an $\ell_{\infty}$ ball and construct a certifiable defense against shadow attacks,
    \item and compare our defense against other optics-based adversarial attacks.
\end{itemize}

\begin{figure*}
  \centering
  \includegraphics[width=0.8\linewidth]{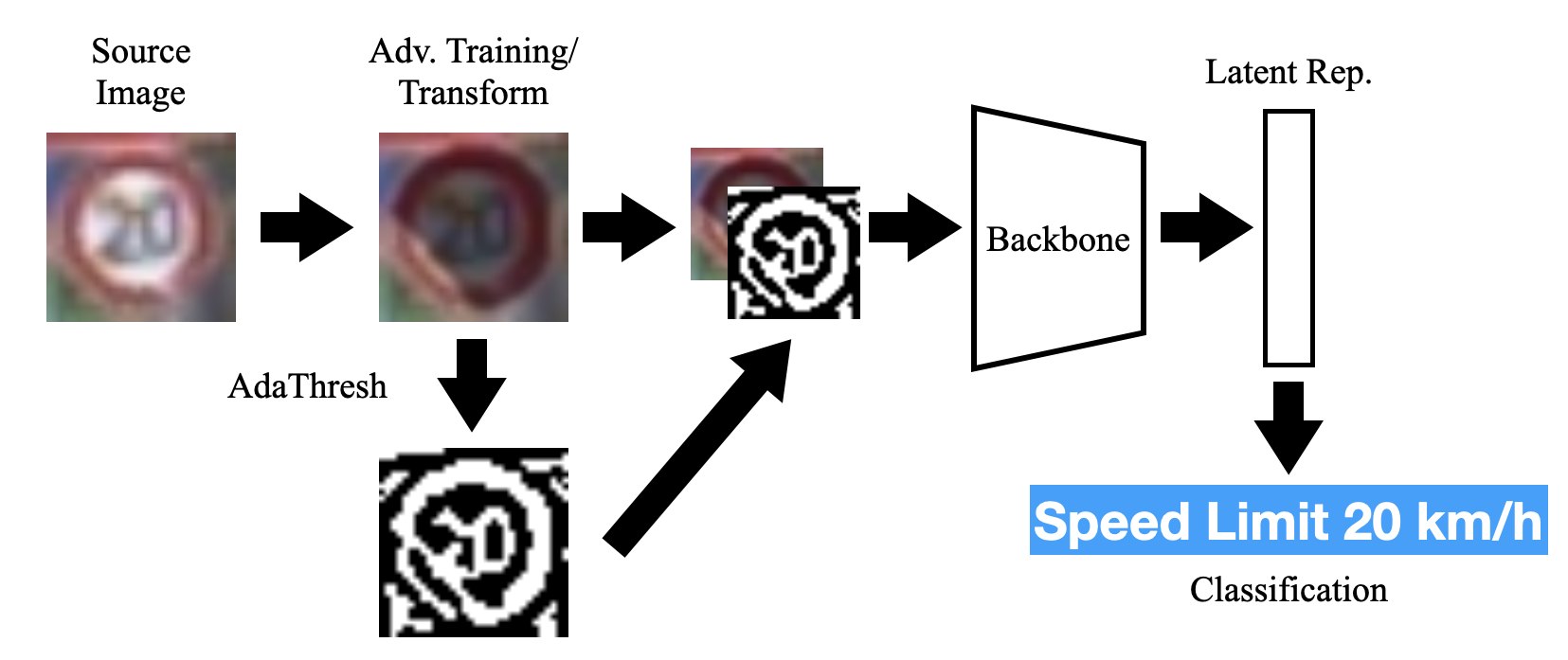}

   \caption{An overview of our defense. We take source images, apply a shadow with parameter $k$ randomly to the image, transform the image with shear, rotation, and translation, and use adaptive thresholding (``AdaThresh") to generate a binary threshold map. The transformed image and the threshold map are concatenated to form a 4-channel image, which the model is retrained on. The rest of the model architecture (``backbone") need not be changed for good robustness. }
   \label{fig:overveiw}
\end{figure*}

\section{Background and Related Work}

Adversarial examples from a general perspective were first introduced by Dalvi \etal \cite{dalvi2004adversarial}, and refined by Szegedy \etal \cite{szegedy2013intriguing} with a constraint that the adversarial examples must be no more than $\varepsilon$ away from the clean example, thereby causing examples to be impercetible to humans for sufficiently small $\varepsilon$.

\subsection{So What Are Adversarial Examples Exactly?}

In this work, we consider a classification setting with $C$ distinct classes, in which we wish to classify an RGB image $x \in \mathbb{R}^n$ by passing $x$ through a model $M$ (such as a neural network) and computing a class label $y = M(x) \in \{1, \ldots, C\}$. 

This setting includes an adversary, who is able to take any such example $x$ and transform it into $x_{\text{adv}}$, a point in the convex set of possible transformed examples $S_{\text{adv}}(x) \subseteq \mathbb{R}^n$. Intuitively, $S_{\text{adv}}(x)$ is a collection of all the examples that, by some metric, are ``close" to $x$ \cite{balunovic2020adversarial}. For instance, if the adversary wishes to keep the transformations within a $\varepsilon$ range $L_p$ perturbation from $x$, then $S_{\text{adv}}(x) = \{x' : ||x' - x||_{p} < \varepsilon\}$.

An adversarial example $x_{\text{adv}}$, then, is an example that is ``close" to $x$ yet causes the model $M$ to misclassify it:
\begin{equation}
    M(x) \neq M(x_{\text{adv}}) \quad x_{\text{adv}} \in S_{\text{adv}}(x).
\end{equation}

Extending this notion, an adversarial attack is a principled manner for an adversary to generate these adversarial examples, \eg, through some mathematical formulation or algorithm.

\subsection{Shadow Attack on Road Signs}

Shadows as an adversarial attack were first proposed by Zhong \etal \cite{zhong2022shadows}, although concern over shadows in road sign recognition have existed for some time \cite{fleyeh2006shadow, li2015novel}. Their attack generates an adversarial image by
\begin{enumerate}
    \item choosing a parameter $k$ representing the ``darkness" or ``strength" of the shadow, where higher values of $k$ indicate weaker shadows and vice versa,
    \item locating a polygon $\mathcal{P}_{\mathcal{V}}$, defined by a set of vertices $\mathcal{V} = \{(m_1, n_1), \ldots, (m_s, n_s)\}$, and a mask $\mathcal{M}$ to locate the target polygon,
    \item converting $x$ from RGB color space to LAB image space, such that each element in $x_{i, j} \in \mathbb{R}^3$ represents the L, A, and B channels respectively,
    \item forming a new image $x_{\text{adv}}$ by recalculating the value of every pixel $(i, j)$ with
    \begin{equation}
        x_{\text{adv},i,j} = \begin{cases}
            x_{i, j} \cdot \scalemath{0.75} {\begin{bmatrix}k & 1 & 1\end{bmatrix}^\top} & (i, j) \in \mathcal{P}_{\mathcal{V}} \cap \mathcal{M} \\
            x_{i, j} \cdot \scalemath{0.75}{\begin{bmatrix}1 & 1 & 1\end{bmatrix}^\top} & (i, j) \not\in \mathcal{P}_{\mathcal{V}} \cap \mathcal{M} \\
        \end{cases},
    \end{equation}
    \item and converting $x_{\text{adv}}$ back into RGB space \cite{zhong2022shadows}.
\end{enumerate}

Finding $\mathcal{V}$ can be formulated as an optimization problem, which Zhong \etal solve by means of Particle Swarm Optimization (PSO) \cite{kennedy1995particle}.

\subsection{Adaptive Thresholding}
In general, \textbf{thresholding} is the process of generating a binary image $b$ (a.k.a. the \textbf{threshold map}) from a source image $s$, where the white pixels are the ``foreground" (255) elements and the black pixels are the ``background" (0) elements. The ``threshold" is the means by which foreground and background are separated. \textbf{Adaptive thresholding} \cite{bradley2007adaptive, white1983image} is a form of local thresholding, where each pixel $b_{i, j}$ is assigned to the foreground or the background by a threshold function $T$ parameterized on pixel coordinates $i, j$:
\begin{equation}
    b_{i, j} = 
    \begin{cases}
        255 & s_{i, j} > T(i, j) \\
        0 & s_{i, j} \leq T(i, j).
    \end{cases}
    \label{eq:adathresh}
\end{equation}
The benefit of adaptive thresholding over global thresholding is that $T$ is parameterized on each pixel, and is thus robust to the spatial changes in the illuminations that represent the variations present in shadow attacks, as opposed to setting one uniform threshold for the entire image.

In our approach, we let $T$ be a Gaussian-window weighted sum of a $k \times k$ neighborhood around $(i, j)$ \cite{SASPWEB2011, stephane1999wavelet}. If $N(i, j)$ contains the set of all points within a $k \times k$ neighborhood around $(i, j)$, then 
\begin{equation}
    T(i, j) = \sum_{(x, y) \in N(i, j)} G_{x, y} \cdot s_{x, y}
\end{equation}
where $G_{i, j}$ is the Gaussian-window (\cref{fig:gaussian}) weight for pixel $(i, j)$, defined as
\begin{equation}
    G_{i, j} = \alpha \cdot \exp\left(\frac{-(i - \frac{k - 1}{2})^2 -(j - \frac{k - 1}{2})^2}{2 \sigma^2}\right)
    \label{eq:gaussian}
\end{equation}
where the standard deviation $\sigma$ is calculated from all pixels in $N(i, j)$ and $\alpha$ is a scaling factor such that the $G_{i, j}$ sum to 1.
$k$ is a tuneable hyperparameter representing the ``aperture" of our local threshold; we chose $k = 3$ to due to the small size of our input images, but the primary focus is on presenting the robustness of the method and not tuning hyperparameters.  

\begin{figure}[t]
  \centering
   \includegraphics[width=0.8\linewidth]{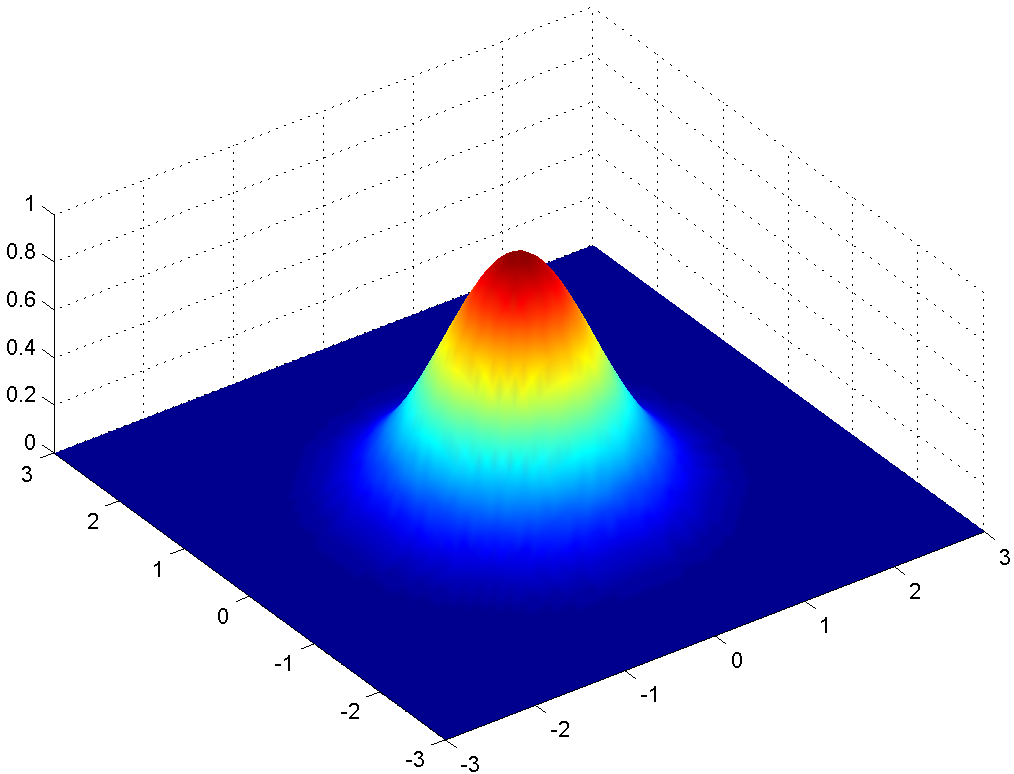}

   \caption{A visualization of $G_{i, j}$ around the point $(i, j) = (0, 0)$ with window size $k = 3$. The motivation behind using Gaussian window weights is that pixels closest to $(i, j)$ receive maximum weights, while pixels further away receive less weight. Image licensed under Creative Commons Attribution-Share Alike 3.0 \cite{gaussian}.}
   \label{fig:gaussian}
\end{figure}

\subsection{Canny Edge Detection}
\label{sec:cannyedge}
Closely related to local thresholding is Canny edge detection, first proposed in 1986 \cite{canny1986computational}. The Canny edge detection algorithm usually takes the following steps \cite{ding2001canny}:
\begin{enumerate}
    \item smooth the image with a Gaussian-window filter (see \cref{eq:gaussian}) to reduce noise,
    \item determine the gradient magnitude and direction at each pixel,
    \item apply a custom thresholding function wherein edge (``foreground") pixels have gradient magnitudes larger than those of its two neighbors in the gradient direction,
    \item and clean up extraneous or ``weak" edges with a hysteresis thresholding step with two threshold parameters, $t_{\text{hi}}$ and $t_{\text{lo}}$. Pixels that are above $t_{\text{hi}}$ become edges, below $t_{\text{lo}}$ are rejected as edges, and in between are edges only if they are connected to an edge pixel.
\end{enumerate}

The output of any edge algorithm \cite{kumar2013algorithm}, Canny included, is referred to as the \textbf{edge map} of the image.

Canny edge detection is both local and global: the first thresholding step is local with respect to the gradient direction, while the second thresholding step is global, as the parameters are set for the whole image. 

While both methods are tested in our approach, we hypothesize that for darker shadows, the greater variance between gradient pixels \textit{inside the shadowed region} and gradient pixels \textit{outside the shadowed region} would cause Canny edge maps to have worse performance than adaptive threshold maps. 
\subsection{Related Work}

While shadows as an adversarial attack were first conceived by Zhong \etal \cite{zhong2022shadows}, shadows have been a source of concern in traffic sign recognition for some time \cite{de2003traffic, fleyeh2011eigen}. However, to our knowledge, the only defense against such attacks remains that proposed by Zhong \etal: a simple adversarial training scheme.

More generally, shadows as an adversarial attack falls into a class of non-invasive optical attacks which are particularly pernicious against self driving cars. While digital attacks like FGSM \cite{huang2017adversarial} and DeepFool \cite{moosavi2016deepfool} are effective in the online domain, such attacks are infeasible in realistic self-driving scenarios, leading to the rise of physical, non-invasive attacks that cause natural or imperceptible perturbations to physical road signs \cite{sayles2021invisible, li2019adversarial, gnanasambandam2021optical, song2018physical}. Zhong \etal noted these attacks relied on sophisticated and complex equipment, making their implementation impractical---for this same reason, it is impractical for us to test our defense on these attacks.

To our knowledge, there is no literature regarding the use of adaptive threshold maps or threshold maps as adversarial defenses. However, edge detection is closely related to thresholding \cite{nadernejad2008edge}. Previous work has investigated edges as an adversarial defense against different classification tasks \cite{ding2019sensitivity, sun2021can}, but their work both use edges as the sole source of information, rather than augmenting it to existing images as in our method. Doing so puts heavily reliance on the quality of edge detectors, which can introduce their own vulnerabilities to adversarial examples if neural networks are used \cite{cosgrove2020adversarial}. 



\section{Our Approach}
\label{sec:ourapproach}
In this section, we give motivations for our defense and empirical justification for those motivations. We then describe our defense and training regime. Our framework supports two defenses: one with adaptive threshold maps and one with edge maps, and requires no more than one modification to the network architecture to achieve good results against shadow attacks.
\subsection{Our Motivation}
We were initially motivated to investigate edge profiling as a possible defense, since
\begin{itemize}
    \item humans recognize road signs based on the boundaries of the sign and the text or symbols within, which can be represented with an edge map \cite{bansal2013edges}, and
    \item convolutional neural networks, which are often used in state-of-the-art (SOTA) road sign recognition, can learn ``unimportant" (from a human perspective) and non-robust features like texture \cite{geirhos2018imagenet}, leaving them vulnerable to adversarial attacks which disturb these features.
\end{itemize}
However, edge detection algorithms are imperfect and edge maps alone may be insufficient for object recognition \cite{sanocki1998edges}. With this in mind, we decided to proceed with appending edge maps of source images as a fourth channel as a way of ``emphasizing" the importance of edges, without destroying the information provided by the source image.

This intuition is empirically justified by experiments with saliency maps \cite{simonyan2013deep, shrikumar2017learning}, a sample of which are presented in \cref{fig:saliency}. These experiments indicate that SOTA sign recognition models, such as those found in Eykholt \etal \cite{eykholt2018robust}, learn features inside the sign but not necessarily the text or boundaries of the sign itself---likely because most signs share one of a few simple edge boundaries, and thus edge maps are shared across many classes in the training set, so the model fails to recognize its robustness. Although sufficient in a benign setting, when an adversary perturbs the sign contents sufficiently, the model fails to ``fallback" and learn from the edges or other low-level features, like the sign's edge boundaries.

The motivation behind thresholding was based on our hypothesis in \cref{sec:cannyedge} and empirical tests that threshold maps appeared more readable and less noisy than edge maps (\cref{fig:threshvsedge}). However, for completeness, we tested both methods. 

\begin{figure}[t]
  \centering
   \includegraphics[width=0.8\linewidth]{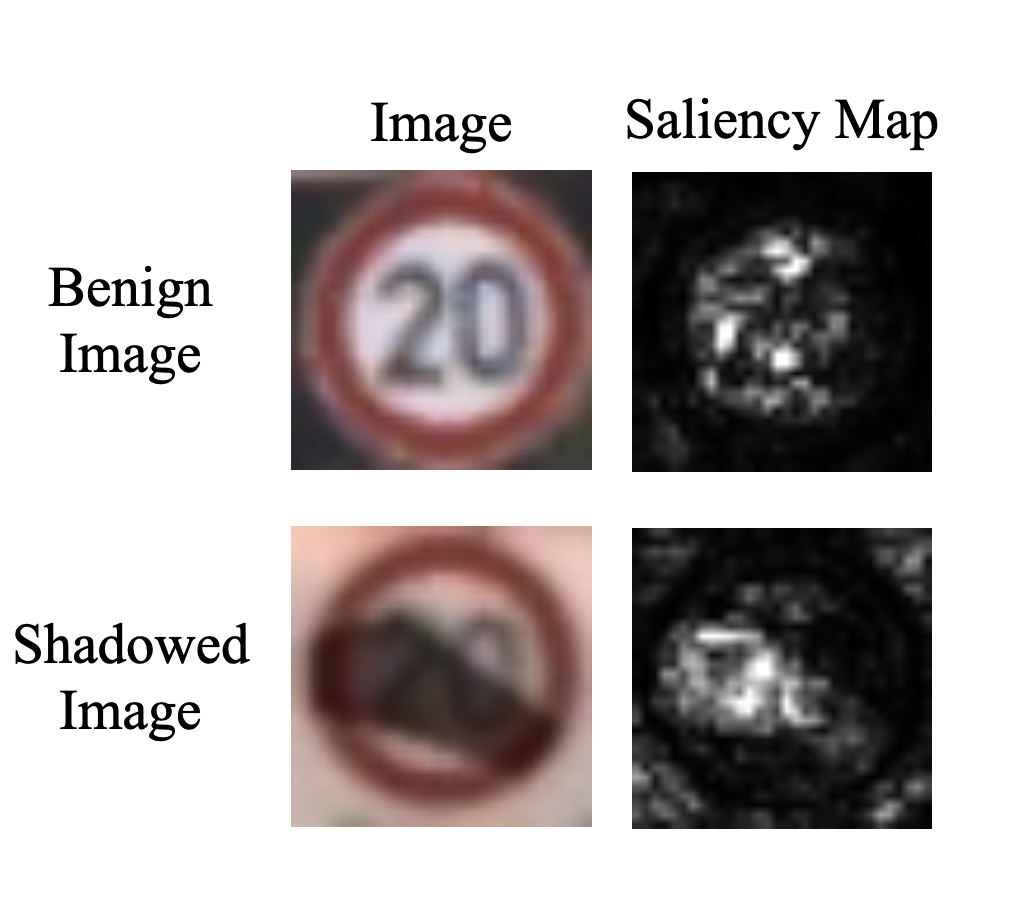}

   \caption{Saliency maps \cite{shrikumar2017learning, uozbulak_pytorch_vis_2022} (right) for a randomly selected benign traffic sign (top left) and its adversarially-shadowed counterparts (bottom left) from the GTSRB \cite{gtsrb} dataset on the model from \cite{eykholt2018robust}. The network appears to learn features inside the sign boundary, leaving it vulnerable when the inside of the sign is perturbed heavily.} 
   \label{fig:saliency}
\end{figure}

\begin{figure}[t]
  \centering
   \includegraphics[width=0.8\linewidth]{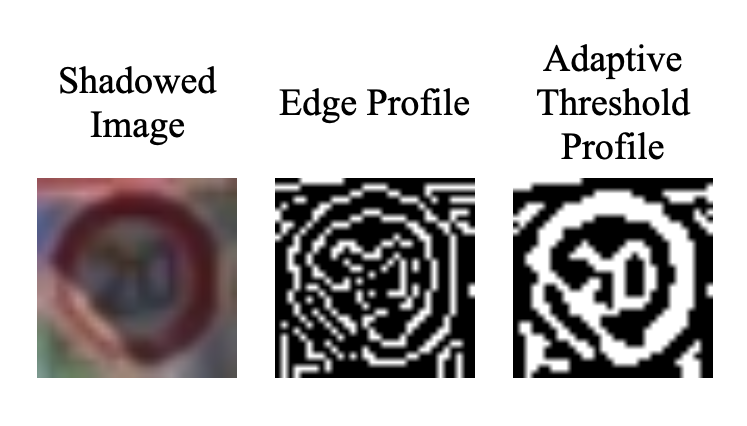}

   \caption{A comparison of a shadowed image (left) with parameter $k = 0.43$, its edge map (middle) and its threshold map (right). The threshold map appears much less noisy, motivating the use of threshold maps.} 
   \label{fig:threshvsedge}
\end{figure}

\subsection{Our Defense}
Our defense is a variant on adversarial training. Given a training dataset $D_{\text{train}}$ and a model $M$, our defense quadruplicates the dataset based on boolean flags \texttt{adv} and \texttt{transform}, which control the addition of adversarial examples and transformed images (shear, rotate, and translate) to the training regime. For each of these dataset duplicates, our defense
\begin{enumerate}
    \item modifies the first layer of $M$ to accomodate a 4-channel image,
    \item takes in a source image $s$,
    \item randomly applies a shadow if \texttt{adv} is \texttt{True},
    \item generates the relevant profile $p$ (threshold or edge map),
    \item appends the profile $p$ to $s$ as a fourth channel, making a new image $x$,
    \item transforms $x$ if \texttt{transform} is \texttt{True}, and
    \item trains the model on $x$. 
\end{enumerate}

If augmenting source images with threshold maps, then we generate an threshold map using \cref{eq:adathresh}. In the case of edge maps, our defense uses Canny edge detection due to its non-differentiablity (compared to a neural network edge mechanism) and speed. The parameters $t_{\text{high}}$ and $t_{\text{low}}$ are chosen based on the formula
\begin{align}
    t_{\text{lo}} &= \max(0, \mu \cdot (1 - \sigma)) \\
    t_{\text{high}} &= \min(255, \mu \cdot (1 + \sigma)) \nonumber
\end{align}
where $\mu$ is the median of the image across all channels and $\sigma$ is is a parameter representing the ``blurring" of the edges: higher sigma indicates more extreme blurring. Intuitively, this formula takes $\sigma$ as a standard deviation and sets the upper and lower thresholds to be one standard deviation from the mean.

In our experimentation, we chose $\sigma = 0.33$ to maintain the 2:1 ratio of $t_{\text{high}}$ to $t_{\text{low}}$ recommended by Canny \cite{canny1986computational}. We also chose $\mu$ as the median and not the mean, since shadows bias the image histogram towards the darker side. 


\section{Experimental Results}

In this section, we present our experiments, namely:
\begin{itemize}
    \item we present robustness results for models using edge maps and adaptive threshold maps, 
    \item we present benign test accuracies for both models,
    \item we reformulate the shadow attack with parameter $k$ as an instance of a $\varepsilon$ perturbations within an $\ell_p$ ball,
    \item using the above reformulation, we present robustness results for both models against the Fast Gradient Sign Method (FGSM), Projected Gradient Descent (PGD), Boundary Attack, and a number of other optics-based adversarial attacks, and
    \item we perform an ablation study on our defense.
\end{itemize}
\subsection{Experimental Setup}
For consistency, we used the same network architectures \cite{eykholt2018robust} and preprocessing steps as Zhong \etal \cite{zhong2022shadows}, with the exception that our networks were modified to take in 4 channels and retrained accordingly.

We also used the same train/test split and the same code as in Zhong \etal \cite{zhong2022shadows} for the shadow attack, and evaluated the robustness accordingly. This includes their exclusion of images which are already too dark; i.e., the mean of their pixel values in the L channel is no larger than 120.

Our experiments were conducted with the PyTorch deep learning library \cite{NEURIPS2019_9015}, on a machine with eight Intel(R) Xeon(R) Bronze 3106 CPU @ 1.70GHz CPUs and one NVIDIA Titan-XP GPU with CUDA version 10.2. The operating system was Ubuntu 18.04 (LTS) with Python version 3.10.5. 

\subsection{Core Robustness Results}
We used \cref{alg:one} as our testing regime, and tested the same $k$ values as Zhong \etal \cite{zhong2022shadows}; namely, the set of values 
\begin{align*}
  \mathcal{K} = \{\,
  &0.20, 0.25, 0.30, 0.35, 0.40, 0.43, 0.45, \\
  &0.45, 0.50, 0.55, 0.60, 0.65, 0.70\}.
\end{align*}
Note that the value $k = 0.43$ is the median shadow value from the SBU Shadow Dataset \cite{vicente2016large}. There are two results we record: the robustness, defined as $1 - \text{attack success rate}$, and the number of queries that the shadow attack makes to the backbone model, which is a measure of the attack's stealthiness. 

Our results are presented in \cref{tab:core} and \cref{tab:core2} for robustness and average number of queries, respectively. In line with our hypotheses in \cref{sec:ourapproach}, we found that while both edge maps and adaptive thresholding provided similar robustness, the quality of edge maps depended greatly on the strength of the shadow, whereas adaptive threshold maps, due to their local nature, were more invariant. 

\begin{algorithm}
\caption{Our Testing Regime}\label{alg:one}
\begin{algorithmic}
\For{dataset in $\{\text{GTSRB, LISA}\}$}
\For{$k$ in $\mathcal{K}$}

\State $\alpha \gets$ test accuracy on benign examples
\State $\beta_1 \gets$ test accuracy on shadowed examples
\State $\beta_2 \gets$ num. queries on shadowed examples
\EndFor
\EndFor
\end{algorithmic}
\end{algorithm}

\begin{table*}
  \centering
  \begin{tabular}{@{}lccccccccccccc@{}}
    \toprule
    Defense/Dataset & $k=$ 0.20 & 0.25 & 0.30 & 0.35 & 0.40 & \textbf{0.43} & 0.45 & 0.50 & 0.55 & 0.60 & 0.65 & 0.70 \\
    \midrule
    Zhong \etal: None & 2.63 & 3.65 & 4.86 & 6.55 & 8.64 & \textbf{9.53} & 11.03 & 12.85 & 15.75 & 19.64 & 26.24 & 33.27 \\
    Zhong \etal: AT & 15.89 & 17.78 & 20.39 & 23.12 & 26.69 & \textbf{28.38} & 30.03 & 30.26 & 33.78 & 43.01 & 48.57 & 55.30 \\
    Ours: AdaThresh & 73.82 & 74.69 & 75.12 & 75.58 & 76.41& \textbf{76.63} & 76.29 & 76.71 & 77.61 & 79.65 & 79.26 & 78.57\\
    Ours: Edges & 75.41 & 75.19 & 76.82 & 76.89 & 77.08 & \textbf{78.02} & 78.12 & 79.33 & 80.59 & 81.97 & 83.55 & 85.46 \\
    \bottomrule
  \end{tabular}
  \caption{The above table describes the average robustness ($1 - $ success rate of attack) over $n = 5$ trials for a specified defense (adversarial training from Zhong \etal, our adaptive threshold maps, or our edge maps) and shadow attack with parameter $k$. While AdaThresh tends to have lower robustness, it has a tighter variance than that of edge map defenses.}
  \label{tab:core}
\end{table*}

\begin{table*}
  \centering
  \begin{tabular}{@{}lccccccccccccc@{}}
    \toprule
    Defense/Dataset & $k=$ 0.20 & 0.25 & 0.30 & 0.35 & 0.40 & \textbf{0.43} & 0.45 & 0.50 & 0.55 & 0.60 & 0.65 & 0.70 \\
    \midrule
    Zhong \etal: AT & 98 & 93 & 112 & 129 & 128 & \textbf{126} & 136 & 155 & 188 & 232 & 249 & 343 \\
    Ours: AdaThresh & 210 & 251 & 256 & 289 & 348 & \textbf{355} & 331 & 329 & 430 & 501 & 574 & 560 \\
    Ours: Edges & 304 & 277 & 370 & 376 & 396 & \textbf{456} & 450 & 515 & 580 & 645 & 726 & 819\\
    \bottomrule
  \end{tabular}
  \caption{The above table describes the average number of queries to the backbone model over $n = 5$ trials for a specified defense (adversarial training from Zhong \etal, our adaptive threshold maps, or our edge maps) and shadow attack with parameter $k$. This is a proxy measure of the ``stealthiness" of the black-box attack. The variance in edge map defenses remains significantly higher than AdaThresh defenses.}
  \label{tab:core2}
\end{table*}

\subsection{Reformulation}
\begin{theorem}
\label{thm:one}
Given an image in RGB color space $x = [R(x), G(x), B(x)]$, its LAB counterpart $x_{\text{lab}} = [L(x), A(x), B(x)]$, and its adversarial counterparts $x_{\text{adv}}$ and $x_{\text{lab, adv}}$, for a shadow attack with parameter $k$, the adversarial perturbation $||x_{\text{adv}} - x|| = \varepsilon_k $ is at most $||M|| \cdot {100|k-1|}$, where $M$ is the matrix multiplication to convert from LAB to RGB space.
\end{theorem}
\begin{proof}
See the Appendix.
\end{proof}

\subsection{Robustness Against Other Attacks}

To test our robustness against other attacks, we first obtained a value for $\varepsilon$ when $k = 0.43$ based on \cref{thm:one}.
\section{Discussion and Limitations}
Beyond providing additional information in the form of a fourth channel, this attack also combines gradient masking (as the binary edge or threshold map is non-differentiable) and adversarial training (as it is trained on shadowed images. However, even though it uses gradient masking, the defense remains robust even to non-gradient based attacks like Boundary Attack and shadow attacks. 

There are a few limitations to this attack. First, its reliance on adversarial retraining makes its training process slow. Second, it uses off-the-shelf edge and adaptive thresholding methods; it is worth investigating a more bespoke set of thresholding equations in future research. Third, its robustness is significantly higher than the baseline, but insufficient for critical tasks like sign recognition; our work should be seen as a starting point and not a final defense. 

\section{Conclusion}
In this paper, we presented a novel and simple adversarial defense against shadow-based adversarial attacks. Our defense requires no retuning or redesign of the model architecture to achieve good robustness against shadow attacks. Furthermore, our defense remains robust against classic adversarial attacks including other optics-based adversarial attacks.
\section{Acknowledgments and Funding}
We would like to thank Dr. Gregory Ditzler, David Schwartz, and Huayu Li for their assistance in developing this paper. This paper was funded by National Science Foundation (NSF) Grant IIS-195039.
{\small
\bibliographystyle{ieee_fullname}
\bibliography{egbib}
}

\end{document}